\documentclass{article} % For LaTeX2e
\usepackage{iclr2027_conference,times}
 
\usepackage{amsmath,amsfonts,bm}

\def\eqref#1{equation~\ref{#1}}
\def\1{\bm{1}}

\def\eps{{\epsilon}}

\DeclareMathAlphabet{\mathsfit}{\encodingdefault}{\sfdefault}{m}{sl}
\SetMathAlphabet{\mathsfit}{bold}{\encodingdefault}{\sfdefault}{bx}{n}

\newcommand{\E}{\mathbb{E}}

\usepackage{hyperref}
\usepackage{url}
\usepackage{amsmath, amssymb, amsthm}
\usepackage{mathtools}
\usepackage{booktabs}
\usepackage{graphicx}
\usepackage{xcolor}
\usepackage{enumitem}
\usepackage{cleveref}
\usepackage{multirow} 
\usepackage{algorithm}
\usepackage{algpseudocode}
\usepackage[most]{tcolorbox}
\usepackage{wrapfig}

\theoremstyle{definition}

\newtheorem{proposition}{Proposition}

\newtheorem{corollary}{Corollary}

\theoremstyle{definition}

\renewcommand{\eps}{\varepsilon}
\providecommand{\E}{\mathbb{E}}

\newcommand{\Dval}{D_{\mathrm{val}}}
\newcommand{\Fcal}{\mathcal{F}}
\newcommand{\ind}{\mathbf{1}}
\newcommand{\Wealth}{W}

\title{When Is Enough Enough in Self-Evolving LLM Systems?}

\author{
Enoch Yin$^{1}$ \qquad
Bin Liu$^{2,*}$ \qquad
Zhengling Qi$^{3,*}$ \\
\small $^{1}$Department of Statistics, The George Washington University \\
\small $^{2}$Department of Statistics and Data Science, School of Management, Fudan University \\
\small $^{3}$Department of Decision Sciences, School of Business, The George Washington University \\
\small $^{*}$Corresponding authors
}

\iclrfinalcopy
\begin{document}

\maketitle
\thispagestyle{plain}
\pagestyle{plain}

% =====================================================================
\begin{abstract}
Self-evolving large language model (LLM) systems repeatedly propose, evaluate, and incorporate updates to prompts, skills, or other persistent artifacts. Despite their growing effectiveness, these systems typically operate under a predetermined iteration or compute budget, without a principled criterion to determine when further evolution is no longer worthwhile. This can lead to two undesirable consequences: unnecessary computation after performance has saturated and the risk of returning late updates that overfit or exploit the evaluation signal. These issues motivate us to study two fundamental questions: \emph{when should a self-evolving system stop, and what should it output once it stops?} We address the first by formulating an online sequential testing problem and constructing an anytime-valid restart detector using the per-item paired evaluation outcomes already produced by self-evolving LLM systems. We address the second by formulating a change-point estimation problem and using the estimated transition to select an earlier artifact for output. The resulting procedure is plug-and-play and requires no modification of the underlying self-evolving algorithms. Across two self-evolving frameworks, three LLM model families, and five benchmarks, our method substantially reduces computation costs while maintaining comparable unseen-test performance. For example, on SearchQA with SkillOpt and DeepSeek V4 Flash, our method stops at round 4 rather than the full budget of 40, reducing token usage by 91.6\% while achieving 82.43\% unseen-test accuracy versus 82.00\% under the full-budget run.
\end{abstract}

% =====================================================================
\section{Introduction}
\label{sec:intro}
LLM agents are increasingly designed to improve themselves through repeated interaction, feedback, and adaptation rather than remain fixed after their initial construction. This emerging paradigm of \emph{self-evolving agents} has appeared in a wide range of settings, including iterative self-correction, self-training and self-play, prompt and skill evolution, memory and experience refinement, and multi-agent co-evolution~\citep{selfrefine, reflexion, ragen, absolutezero, webrl, agent0, multiagentevolve, skillopt, gepa, textgrad, trace2skill, autoskill, memskill, arise, skillflow, skill1}. Despite substantial differences in what is evolved and how updates are generated, these systems share a similar iterative structure: the system generates or applies an update, evaluates its effect using feedback, reward, or validation signals, and carries the resulting state into subsequent rounds. 

In practice, self-evolving loops are typically run for a fixed number of iterations or until a pre-specified compute budget is exhausted, regardless of whether further evolution remains meaningful. This can have two undesirable consequences. First, once meaningful improvement has saturated, the system may continue proposing, evaluating, and committing updates that provide little or no additional benefit, incurring substantial unnecessary computation. Second, repeatedly adapting to the same validation signal can encourage evaluation overfitting or reward hacking~\citep{amodei2016, skalse2022}, so later updates may appear favorable under the evaluation signal without representing genuine improvement. Figure~\ref{fig:teaser} illustrates how these two consequences can compound in practice: the evolution continues well after meaningful gains have saturated, consuming substantial additional computation, while continued adaptation can ultimately lead to a degraded final artifact on unseen test data.

This limitation raises two fundamental questions that every self-evolving system must answer:

\begin{tcolorbox}[
    colback=gray!4,
    colframe=gray!35,
    boxrule=0.6pt,
    arc=2pt,
    left=8pt,
    right=8pt,
    top=7pt,
    bottom=7pt
]
\centering
\textbf{(i) Given a fixed dataset, when should we stop running the self-evolving loop?}\\[4pt]
\textbf{(ii) Once we stop, which artifact/agent should we output?}
\end{tcolorbox}
The first question asks when there is sufficient evidence that further evolution is no longer worthwhile. The second is distinct: even after deciding to stop, the most recent artifact may not be the best one to return and one needs to decide what to output. 
Current self-evolving systems generally leave both decisions unaddressed in a fixed iteration budget.
\begin{figure}[t]
\centering
\includegraphics[width=0.88\textwidth]{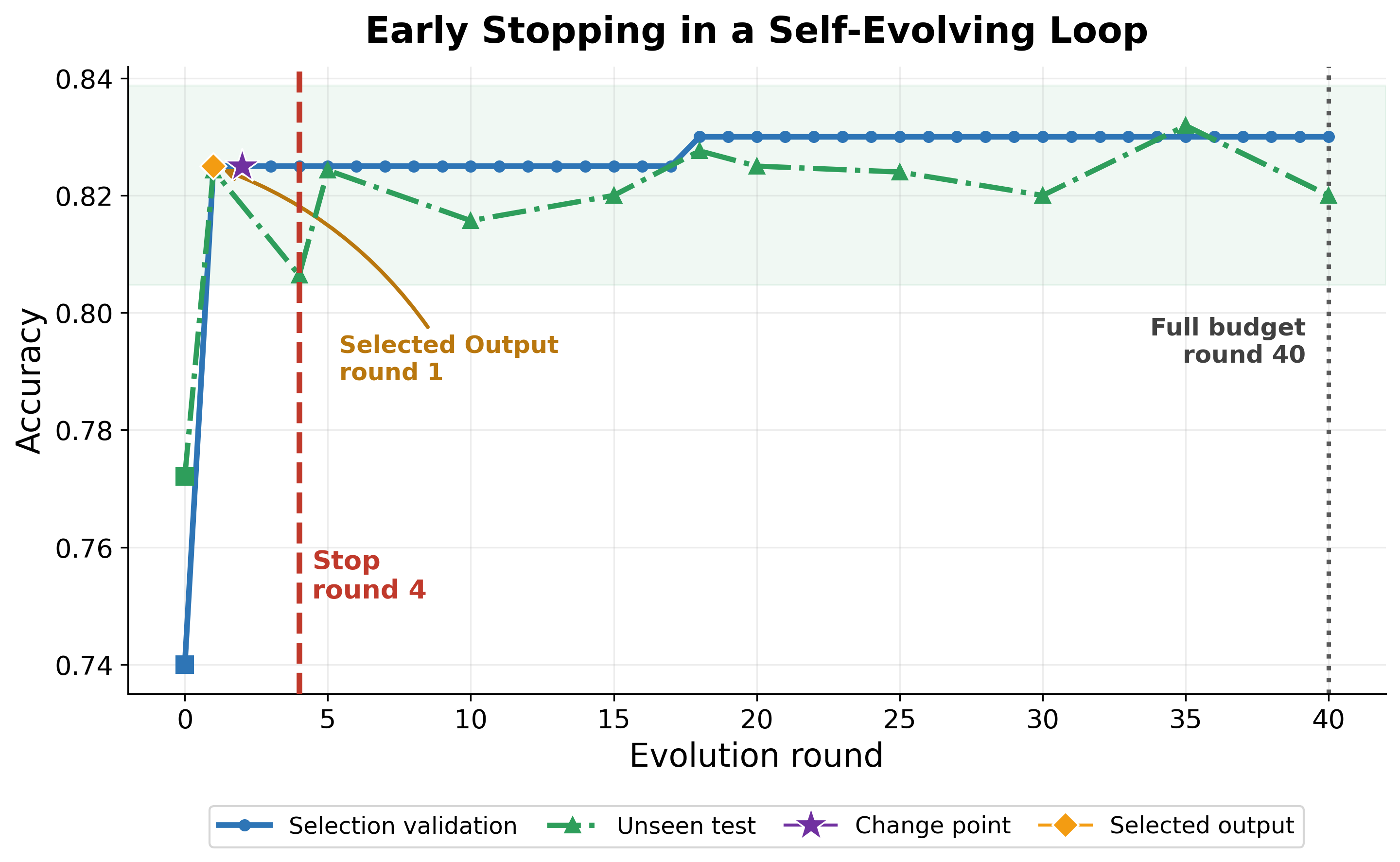}
\caption{A self-evolving run on SearchQA with DeepSeek V4 Flash. Performance improves during the first few rounds and then plateaus, yet the fixed-budget loop continues to round~40. Our procedure stops at round~4 and selects round~1 as the output, achieving comparable unseen-test performance while avoiding most of the subsequent computation.} 
\label{fig:teaser}
\end{figure}

We answer these two questions with a unified online procedure using the per-item paired outcomes already produced by the evaluation gate in most existing systems. For \emph{when to stop}, we formulate sequential hypothesis tests and construct an anytime-valid detector that signals when evidence for continued meaningful improvement has weakened, building on testing by betting~\citep{shafer2021}, betting martingales~\citep{waudbysmith2024}, and e-detectors~\citep{shin2023}. For \emph{what to output}, we formulate a change-point estimation problem~\citep{page1954, lorden1971}: after stopping, we estimate when meaningful improvement ended and return an earlier artifact. The procedure requires no modification to the underlying self-evolving system and uses only outcomes it already computes for incumbents and candidates. In \Cref{fig:teaser}, for example, it stops at round~4, estimates the change point at round~2, and returns an earlier artifact from round~1 with comparable unseen-test performance to the full 40-round run while using $91.6\%$ fewer tokens (\Cref{tab:main}). 

In detail, we make three main contributions. First, we identify \emph{when to stop} and \emph{what to output} as the two fundamental decisions a self-evolving loop must make, and show that both can be addressed online using the per-item paired outcomes already produced by the system, without modifying the underlying self-evolving algorithm (\Cref{sec:problem}). Second, we develop a unified statistical procedure that combines an anytime-valid sequential stopping rule with change-point-based output selection: the stopping rule determines when further improvement is no longer worthwhile, while the estimated change point determines which earlier artifact should be returned (\Cref{sec:proposed}). Third, we evaluate the procedure across two self-evolving frameworks, SkillOpt~\citep{skillopt} and GEPA~\citep{gepa}, three model families, and five benchmarks, showing substantial reductions in computation while maintaining comparable unseen-test performance (\Cref{sec:main}). In particular, in a benchmark where continued evolution eventually produces a reward-hacking update, our procedure stops before this late degraded artifact is reached and returns an earlier artifact instead. 
\section{Related work}
% ===============================================
\textbf{Self-evolving LLM systems.}
A rapidly growing body of work lets an LLM improve itself by repeatedly editing a persistent artifact and evaluating it on held-out data: skill documents \citep{skillopt, trace2skill, autoskill, memskill, evoskill, skillflow, skill1, arise}, prompts \citep{gepa, textgrad}, or programs. These systems differ in what they evolve, but share an analyze--propose--gate loop and report strong gains from it. Their reliability has two known failure points. The loop can be gamed: because it optimizes against its own evaluation signal, degenerate solutions that exploit a leaky benchmark score well without real capability. This is reward hacking, identified as a core safety problem \citep{amodei2016}, formalized by \citet{skalse2022}, catalogued across environments \citep{krakovna2020}, and shown to worsen with optimization pressure and agent capability \citep{pan2022, gao2023}; it is now documented in self-evolving agents \citep{shao2025}. And the per-round commit can be spurious, which \citet{pace} addresses by recasting the accept-or-reject gate as a sequential test. This line of work sharpens the \emph{individual steps} of the loop, its proposer and its gate, but takes the loop's overall length and its final output as given, leaving open when the loop should stop and which of its artifacts to return.

\textbf{Sequential change-point detection and optimal stopping.} Detecting online when a stream's distribution has changed is classical, from Wald's SPRT \citep{wald1945} and Page's CUSUM \citep{page1954} to the Shiryaev--Roberts procedure \citep{shiryaev1963, roberts1966} and its optimality theory \citep{lorden1971, pollak1985, lai1998, tartakovsky2014}. These classical procedures typically rely on specified parametric models for the pre- and post-change distributions, whereas our setting is defined through a conditional mean criterion with an unknown change time. We therefore build on anytime-valid inference \citep{shafer2021, ramdas2023savi}, specifically the betting martingales of \citet{waudbysmith2024} and the e-detectors of \citet{shin2023}, which provide sequential validity under optional stopping. In Appendix~\ref{app:related-work}, we review more related work on stopping rules for training and inference.
% =====================================================================
\section{Method}
\label{sec:method}
\subsection{Problem Formulation}
\label{sec:problem}

We consider a self-evolving system with initial artifact $\theta_0\in\Theta$, where $\Theta$ may contain prompts, skill documents, or programs. At each round $t\ge1$, the system uses its current incumbent $\theta_{t-1}$ and accumulated information to propose a candidate $\theta'_t$. The evolution mechanism may use training data, search, or reflection to generate a candidate $\theta'_t$. The incumbent and candidate are evaluated on a fixed validation set consisting of $n$ tasks: $\Dval=\{x_i\}_{i=1}^n$, with binary correctness $\ell(\theta,x_i)\in\{0,1\}$ and validation accuracy $L(\theta)=n^{-1}\sum_{i=1}^n\ell(\theta,x_i)$. A greedy gate sets $\theta_t=\theta'_t$ if $L(\theta'_t)>L(\theta_{t-1})$ and retains $\theta_t=\theta_{t-1}$ otherwise. Without a stopping rule, evolution continues until a prescribed round budget $N$ or compute budget is exhausted. One naive approach is to monitor $L(\theta_t)$ and detect the change point. However, tracking only $L(\theta_t)$ is insufficient because rejected proposals leave it unchanged, making a flat trajectory uninformative about proposal quality. For example, an apparent plateau may later be followed by further improvement, as in \Cref{fig:spreadsheet}. We therefore accumulate evidence from candidate--incumbent paired outcomes for sequential detection.

\begin{figure*}[t]
    \centering
    \includegraphics[width=\textwidth]{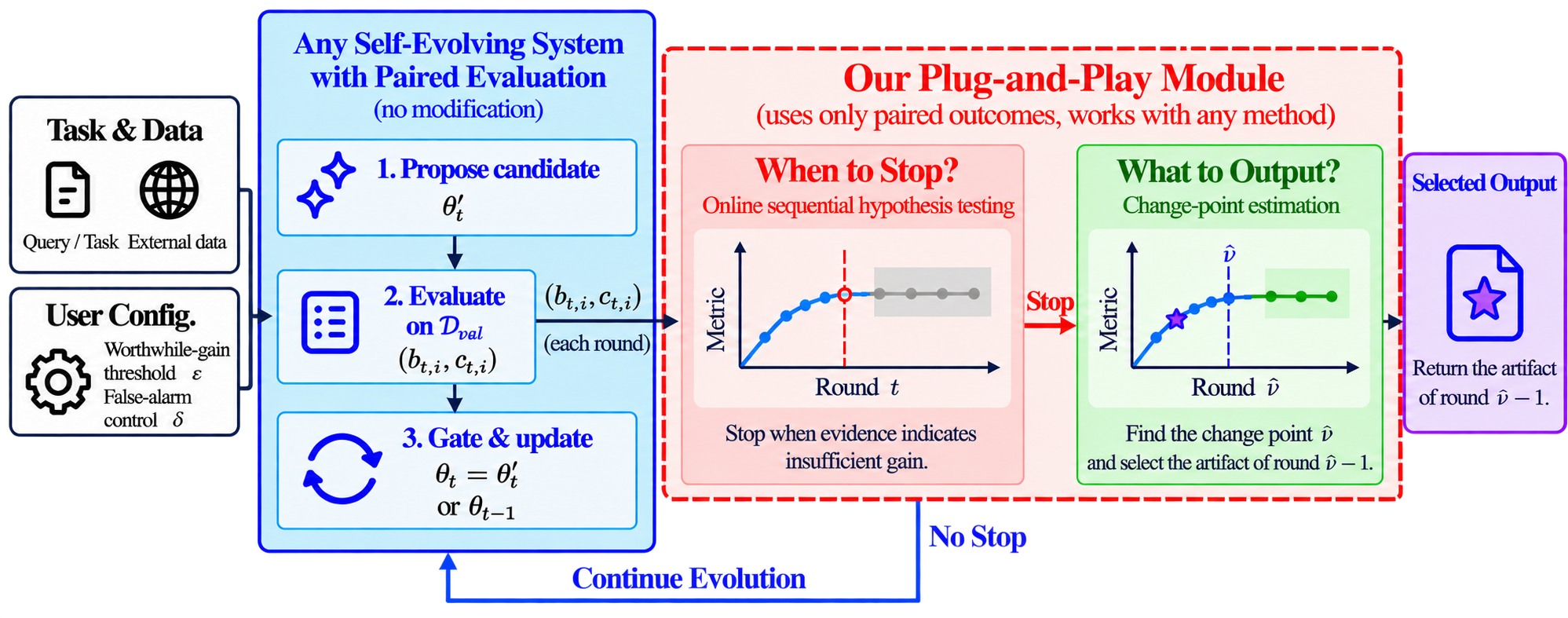}
    \caption{Our stopping and output-selection Plug-and-Play module monitors paired validation outcomes from a self-evolving loop.}
    \label{fig:pipeline}
\end{figure*}

To monitor whether proposed updates continue to provide worthwhile gains, let $b_{t,i}=\ell(\theta_{t-1},x_i)$ and $c_{t,i}=\ell(\theta'_t,x_i)$ denote the incumbent and candidate outcomes on validation task $x_i$. Define the paired differences and their round-wise average as
\begin{equation}
X_{t,i}=c_{t,i}-b_{t,i}\in\{-1,0,+1\},
\qquad \text{and} \qquad
Z_t=\frac{1}{n}\sum_{i=1}^n X_{t,i}
=L(\theta'_t)-L(\theta_{t-1}).
\end{equation}
Each difference $X_{t,i}$ records whether the candidate worsens, preserves, or improves an outcome, regardless of whether it is accepted. Let $\varepsilon>0$ denote the minimum expected proposal gain considered worthwhile. The gate accepts any positive observed gain, while stopping asks whether the conditional expected proposal gain remains at least $\varepsilon$. Thus, a candidate may be accepted even when its expected gain is too small to justify another round. Writing $\Fcal_{t-1}$ for the information available after the round-$t$ candidate is proposed but before its validation outcomes are observed, we define the pre-change condition as
\begin{equation}
\label{eq:null}
\E[Z_t\mid\Fcal_{t-1}]\ge\varepsilon.
\end{equation}
The round average measures the proposed validation gain, while the paired outcomes retain the item-level information used by our proposed sequential tests in \Cref{sec:proposed}.

We seek to determine both when to stop and which retained artifact to return. Let $\nu$ be the first round from which $\E[Z_t\mid\Fcal_{t-1}]<\varepsilon$, and assume this low-improvement regime persists once entered. Our first objective is to detect this transition online by accumulating evidence against the pre-change condition and stopping at an alarm time $T_{\mathrm{alarm}}$, defined by the sequential rule in \Cref{sec:proposed}. Because evidence takes time to accumulate, the alarm can follow the transition by several rounds. The pre-transition output is $\theta_{\nu-1}$, the final retained artifact before expected proposal improvement becomes non-worthwhile. Our second objective is therefore to estimate $\nu$ and return $\theta^\ast=\theta_{\nu-1}$, the incumbent immediately preceding the estimated transition. This separates the time spent gathering evidence from the round whose artifact is selected for output. \Cref{fig:pipeline} shows how our stopping and output-selection module interfaces with this loop.

\subsection{Proposed Method}
\label{sec:proposed}

We now develop a sequential procedure based on the paired improvement signals $\{X_{t,i}\}_{i=1}^n$ defined in \Cref{sec:problem}. We use these signals to address the two problems formulated above. First, we construct sequential tests of the pre-change condition from different starting rounds and combine their evidence to determine when evolution should stop. Second, we use the same restart statistics to estimate when the non-worthwhile-improvement regime began and determine which earlier incumbent should be returned.

\subsubsection{When to Stop: Sequential Testing with Restart Processes}
\label{sec:stopping}

A weak proposal can occur even while further evolution remains worthwhile. We therefore accumulate evidence against continued worthwhile improvement using successive paired outcomes defined before. Following \citet{shin2023}, for each starting round $s$, we test $H_0^{(s)}$: the pre-change condition in \Cref{eq:null} holds for every subsequent round $u\ge s$.

We represent this evidence by a restart wealth process with betting fractions $\lambda_{s,u}$,
\begin{equation}
\label{eq:wealth}
W_t^{(s)}
=
\prod_{u=s}^{t}
\prod_{i=1}^{n}
\left[
1+\lambda_{s,u}\left(\varepsilon-X_{u,i}\right)
\right],
\qquad \text{with} \qquad
W_{s-1}^{(s)}=1.
\end{equation}
For positive betting fractions, an outcome $X_{u,i}<\varepsilon$ contributes a factor greater than one, while an outcome above $\varepsilon$ contributes a factor below one. Multiplication combines these contributions across items and rounds, so a large wealth represents accumulated evidence against $H_0^{(s)}$.

We use two betting schemes for $\lambda_{s,u}$. Under adaptive aGRAPA~\citep{waudbysmith2024}, $\lambda_{s,t}$ is chosen predictably from past data. See Appendix~\ref{app:betting} for more details. Under the fixed
construction, we set $\lambda_{s,t}\equiv\lambda$ for a constant $\lambda\in\Lambda$, and form a fixed mixture by averaging the resulting wealth processes over $\lambda\in\Lambda$. We use $\Lambda=\{0.1,0.2,0.4\}$. In the main paper, we report experimental results under adaptive aGRAPA and  we evaluate both schemes throughout our experiments in Appendix~\ref{app:full-alarms}. 

A process started during useful early rounds may lose wealth and then need to recover those losses before detecting a later stall. We therefore start a fresh process at every round, allowing evidence after a transition to accumulate without losses from earlier gains. Following the CUSUM-style e-detector construction of \citet{shin2023}, we retain the strongest current evidence across starting rounds and raise an alarm when it crosses $1/\delta$, for a user-specified $\delta\in(0,1)$:
\begin{align}
\label{eq:restart}
M_t &= \max_{1\le s\le t}W_t^{(s)}, \qquad \text{and} \\
\label{eq:alarm}
T_{\mathrm{alarm}} &= \inf\left\{t\ge1:M_t\ge\frac{1}{\delta}\right\}.
\end{align}
The threshold determines how much evidence is required to stop: smaller $\delta$ requires larger wealth and makes the rule more conservative. Under the pre-change condition and the assumptions in \Cref{sec:theory}, $M_t$ is a valid e-detector, so this threshold guarantees an average run length to false alarm of at least $1/\delta$.

\begin{algorithm}[t]
\caption{Stopping and output selection for a self-evolving loop}
\label{alg:detector}
\begin{algorithmic}[1]
\Require Worthwhile-gain threshold $\varepsilon$, false-alarm control parameter $\delta$.
\State Initialize restart wealth processes.
\For{$t=1,2,\dots, N$}
    \State Observe $(b_{t,i},c_{t,i})_{i=1}^{n}$ and set $X_{t,i}=c_{t,i}-b_{t,i}$.
    \State Initialize a new restart at $s=t$ and update all wealths with aGRAPA or the fixed mixture.
    \State $M_t\gets\max_{1\le s\le t}W_t^{(s)}$.
    \If{$M_t\ge1/\delta$}
        \State $T_{\mathrm{alarm}}\gets t$.
        \State $\hat{\nu}\gets\operatorname*{arg\,max}_{1\le s\le t}W_t^{(s)}$.
        \State \textbf{stop and return} $\theta_{\hat{\nu}-1}$.
    \EndIf
\EndFor
\State If no alarm is raised, the stopping module remains inactive and the underlying system returns its final artifact.
\end{algorithmic}
\end{algorithm}

\subsubsection{What to Output: Change-Point Estimation}
As discussed above, restarting the sequential evidence process at each round helps the procedure respond more quickly once improvement becomes non-worthwhile. The resulting collection of restart processes also provides a natural way to locate when this transition occurred. A process started too early includes pre-transition observations that tend to work against wealth growth, whereas a process started too late misses part of the evidence accumulated after the transition. This suggests that a process started near the transition should attain relatively large wealth at the alarm time. We therefore estimate the change point as the starting round of the restart process with the largest observed wealth at $T_{\mathrm{alarm}}$:
\begin{equation}
\label{eq:changepoint}
\hat{\nu}
=
\operatorname*{arg\,max}_{1\le s\le T_{\mathrm{alarm}}}
W_{T_{\mathrm{alarm}}}^{(s)}.
\end{equation}
As shown in \Cref{sec:theory}, this intuition can be formalized in the common fixed-$\lambda$ setting. We then use the estimated change point $\hat{\nu}$ to determine the returned artifact according to $\theta^\ast$. Thus, the same restart construction is used both to stop the evolution and to select the output artifact: $T_{\mathrm{alarm}}$ determines when to stop, while $\hat{\nu}$ determines which artifact to return.

Algorithm~\ref{alg:detector} summarizes the complete stopping and output-selection procedure.

\subsection{Theoretical Analysis}
\label{sec:theory}

We establish stopping validity and change-point guarantees below. All proofs are in Appendix~\ref{app:proofs}.

\textbf{Stopping validity.} Recall that $\Fcal_{t-1}$ includes the proposed candidate but not its validation outcomes. Assume that $X_{t,1},\ldots,X_{t,n}$ are conditionally independent given $\Fcal_{t-1}$, and each betting fraction $\lambda_{s,t}\in[0,\tfrac12]$ is predictable and fixed across items within round $t$.

\begin{proposition}[Restart and stopping validity]
\label{prop:validity}
\label{prop:arl}
Under these assumptions, for each start $s$, $H_0^{(s)}$ implies that $W_t^{(s)}$ is a nonnegative supermartingale initialized at one and $
\Pr(\sup_{t\ge s}W_t^{(s)}\ge 1/\delta)\le \delta$
for $\delta\in(0, 1)$. If the pre-change condition holds at every round, then $M_t$ is a valid e-detector: $\E[M_\tau]\le\E[\tau]$ for every stopping time $\tau$. Consequently, $\E[T_{\mathrm{alarm}}]\ge\frac{1}{\delta}.$
\end{proposition}
Proposition \ref{prop:arl} covers adaptive betting including aGRAPA, while fixed mixtures inherit the same validity by convexity, with average time to false alarm at least $1/\delta$. Intuitively, the detector stops only when enough evidence accumulates that further proposal gains are no longer worthwhile. A smaller $\delta$ raises the alarm threshold, increasing the lower bound on $\E[T_{\mathrm{alarm}}]$ under the pre-change regime and lowering the probability of making the false alarm.

\textbf{Change-point identification and localization.}
For this analysis, assume one common fixed betting fraction $\lambda\in(0,\tfrac12]$ across all rounds and restarts. Define
\begin{equation}
\label{eq:normalized-log-growth}
V_{t,i}=\log\!\left(1+\lambda(\varepsilon-X_{t,i})\right),
\qquad
Y_t=\frac1n\sum_{i=1}^n V_{t,i}, \qquad \text{and} \qquad \mu_t=\E[Y_t\mid\Fcal_{t-1}].
\end{equation}

\begin{proposition}[Change-point identification and localization]
\label{prop:changepoint}
\label{prop:changepoint-error}
Assume that (a) $\nu$ marks the sign change in fixed-$\lambda$ expected log growth: $\mu_t<0$ for $t<\nu$ and $\mu_t>0$ for $t\ge\nu$; Then, for any fixed $T\ge\nu$,
\[
\nu=\operatorname*{arg\,max}_{1\le s\le T}\E[\log W_T^{(s)}].
\]
 For localization, further assume that (b) $\mu_t\le-\kappa$ before $\nu$ and $\mu_t\ge\kappa$ from $\nu$ onward for some $\kappa>0$; (c) Conditioning on $\Fcal_{t-1}$, the $V_{t,i}$ are independent and each centered increment $V_{t,i}-\E[V_{t,i}\mid\Fcal_{t-1}]$ is $\sigma^2$-sub-Gaussian ($\sigma>0$). 
 Then, for any $\alpha\in(0,1)$, with probability at least $1-\alpha$, the bound holds uniformly over all finite $T\ge\nu$, including $T_{\mathrm{alarm}}$ on $\{\nu\le T_{\mathrm{alarm}}<\infty\}$:
\[
|\hat\nu_T-\nu|
<\frac{2\sigma^2}{n\kappa^2}
\log\!\left(
\frac{2}{\alpha\left(1-e^{-n\kappa^2/(2\sigma^2)}\right)}
\right).
\]
\end{proposition}
Proposition \ref{prop:changepoint-error} justifies selecting the largest observed restart wealth for localizing the change point. It further shows that for fixed $\kappa$ and $\sigma$, our change point estimation error diminishes at the rate of $1/n$.
%=====================================================================
\section{Experiments}
\label{sec:experiments}
\subsection{Setup}
\label{sec:setup}

We evaluate our method with SkillOpt and GEPA across DeepSeek V4 Flash, Qwen3-32B, and GPT-5.6 Luna, using SearchQA, GSM8K, OfficeQA, LiveMath, and SpreadsheetBench. Main experiments fix SkillOpt and DeepSeek V4 Flash while varying the benchmark; ablations vary the model, evolving system, validation-set size, and detector construction.

Unless otherwise noted, we use $\varepsilon=0.01$ and false-alarm control parameter $\delta=0.05$; $t=0$ denotes the initial artifact. For $\lambda_{s,u}$ in the construction of the wealth process, we run both aGRAPA and the fixed-mixture detector with $\Lambda=\{0.1,0.2,0.4\}$. Since the two betting schemes produce the same estimated change point $\hat\nu$ across matched settings, with differences appearing only in the alarm round, we report aGRAPA in the main results for conciseness and provide the comparison in Appendix~\ref{app:full-alarms}. We report the alarm round $T_{\mathrm{alarm}}$, estimated change point $\hat\nu$, computation consumed through the alarm relative to the full run, and returned artifact $\theta_{\hat\nu-1}$. When an unseen test set is available, we compare the returned artifact with the one selected by the full run on the same examples, reporting their accuracies, paired accuracy difference, and its 95\% confidence interval.

\subsection{Main Results}
\label{sec:main}
\textbf{Our method stops substantially before the fixed budget.} The detector saves 48.4\%--91.6\% of tokens across the three main benchmarks (\Cref{tab:main}). The evolution trajectories and detailed stopping behavior are provided in Appendix~\ref{app:main-trajectories}.

\begin{wrapfigure}{r}{0.48\textwidth}
\centering
\includegraphics[width=\linewidth]{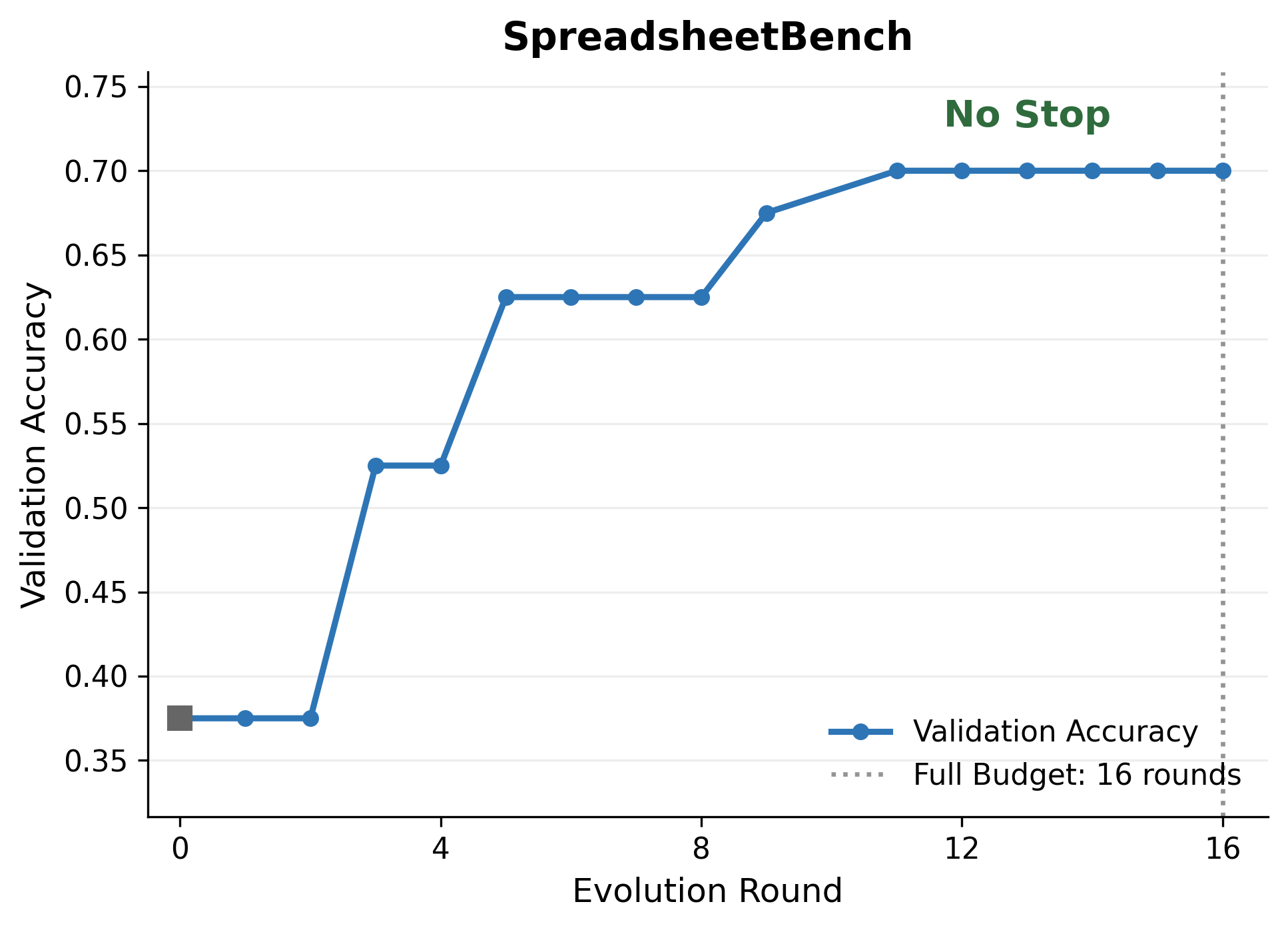}
\caption{SpreadsheetBench trajectory under SkillOpt with DeepSeek V4 Flash.}
\label{fig:spreadsheet}
\end{wrapfigure}
\textbf{Early stopping retains comparable unseen-test performance while improving token efficiency.} After the detector stops, our method uses the estimated change point $\hat{\nu}$ to return the incumbent immediately preceding the estimated transition, $\theta_{\hat{\nu}-1}$. We compare this returned artifact with the artifact selected by the full run on the same unseen-test examples. The paired accuracy difference is defined as $\Delta_{\mathrm{acc}}=\mathrm{Acc}_{\mathrm{returned}}-\mathrm{Acc}_{\mathrm{full}}$, and we report its 95\% confidence interval. The returned artifact slightly outperforms the full-run selection on SearchQA ($+0.43$ pp), while the differences on GSM8K and OfficeQA are $-0.80$ pp and $-1.16$ pp, respectively. In all three cases, the confidence interval includes zero, and the observed differences are small relative to their uncertainty, indicating no clear loss in unseen-test performance from stopping early. At the same time, the reduction in computation substantially improves the cost--performance tradeoff. Measured as millions of tokens per percentage point of unseen-test accuracy gained over the baseline, the returned artifacts are $12.9\times$, $9.3\times$, and $1.8\times$ more token-efficient than running the full evolution budget on SearchQA, GSM8K, and OfficeQA, respectively. The gains are largest on SearchQA and GSM8K, where most of the unseen-test improvement occurs early. Overall, the stopping rule substantially reduces evolution cost while maintaining comparable unseen-test performance in our experiments.

\begin{wrapfigure}{r}{0.48\textwidth}
\centering
\includegraphics[width=\linewidth]{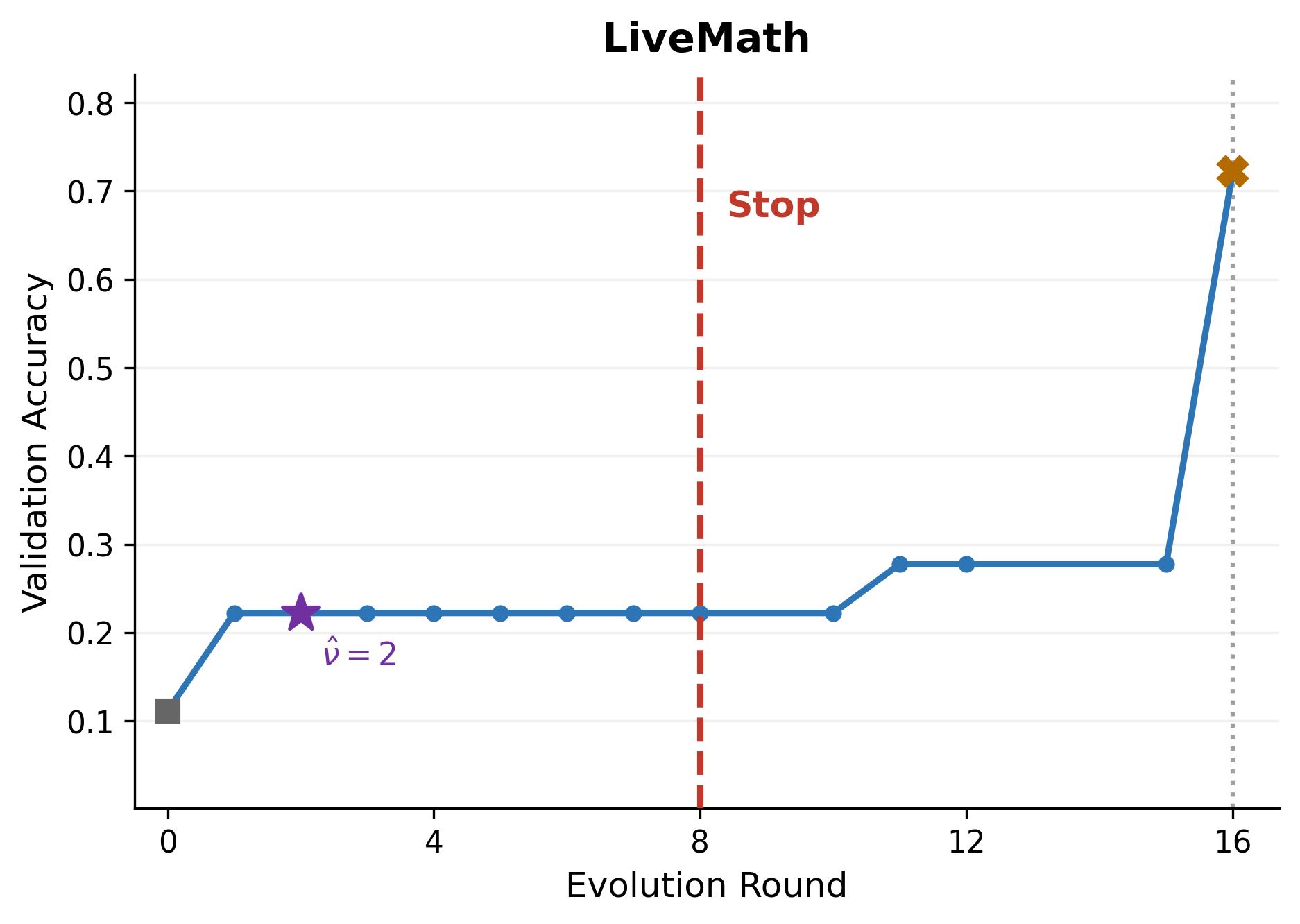}
\caption{LiveMath trajectory under SkillOpt with DeepSeek V4 Flash.}
\label{fig:livemath}
\end{wrapfigure}

\textbf{The detector stays silent when improvement continues.} A useful stopping rule should remain inactive while meaningful improvement continues. SpreadsheetBench serves as a negative control: its validation accuracy rises from $0.375$ to $0.70$ over~$16$ rounds, with further gains as late as round~$11$ (\Cref{fig:spreadsheet}). Under the reported aGRAPA configuration ($\varepsilon=0.01$, $\delta=0.05$), no alarm is raised within the observed horizon. This complements the main results by showing that the detector remains silent on a trajectory that continues to improve.

\textbf{Early stopping avoids a degraded late artifact. }LiveMath further illustrates the value of combining principled stopping with change-point-based output selection. As shown in \Cref{fig:livemath}, SkillOpt makes an early genuine improvement and then enters a long low-improvement regime. If the system were allowed to run for the full budget, it would eventually reach round~16, where it accepts a degenerate skill that instructs the model to answer ``A'' without reasoning, exploiting the benchmark's skewed answer distribution and producing a large accuracy increase. Our detector instead raises an alarm at $T_{\mathrm{alarm}}=8$, stopping the evolution well before this late artifact is ever reached. It also estimates the change point at $\hat{\nu}=2$, and the output-selection rule returns the artifact from round 1. Thus, the stopping rule prevents the system from continuing into the later degraded regime, while the change-point estimate determines which earlier artifact should be returned. The detector does not explicitly identify reward hacking; it stops because sustained meaningful improvement has already ceased.

\begin{table}[t]
\centering
\small
\setlength{\tabcolsep}{4pt}
\renewcommand{\arraystretch}{0.88}
\caption{Main results for DeepSeek V4 Flash + SkillOpt with a shared detector configuration ($\varepsilon=0.01$, $\delta=0.05$, aGRAPA).}
\label{tab:main}
\begin{tabular}{lccc}
\toprule
Benchmarks & \textbf{SearchQA} & \textbf{GSM8K} & \textbf{OfficeQA} \\
\midrule
\multicolumn{4}{l}{\textit{Stopping behavior}} \\
Validation size $n$ & 200 & 200 & 24 \\
Full-run budget $N$ & 40 & 40 & 16 \\
Alarm round $T_{\mathrm{alarm}}$ & 4 & 4 & 8 \\
Change point $\hat{\nu}$ & 2 & 2 & 5 \\
\midrule
\multicolumn{4}{l}{\textit{Computation cost}} \\
Full-run tokens (M) & 25.85 & 11.86 & 206.90 \\
Tokens to alarm (M) & 2.18 & 1.21 & 106.72 \\
Tokens saved & \textbf{91.6\%} & \textbf{89.8\%} & \textbf{48.4\%} \\
\midrule
\multicolumn{4}{l}{\textit{Unseen-test performance}} \\
Test samples & 1400 & 500 & 172 \\
Baseline accuracy (\%) & 77.21 & 84.00 & 14.53 \\
Full-run selected (\%) & 82.00 & 98.20 & 33.72 \\
Returned accuracy (\%) & 82.43 & 97.40 & 32.56 \\
$\Delta$ Accuracy (pp) & +0.43 & -0.80 & -1.16 \\
95\% CI of $\Delta$ (pp) & $[-1.10,\,1.96]$ & $[-1.91,\,0.31]$ & $[-8.72,\,6.39]$ \\
\midrule
\multicolumn{4}{l}{\textit{Token efficiency}} \\
Full run: Tok / 1\% gain (M) & 5.39 & 0.84 & 10.78 \\
Returned: Tok / 1\% gain (M) & 0.42 & 0.09 & 5.92 \\
Efficiency gain & \textbf{12.9$\times$} & \textbf{9.3$\times$} & \textbf{1.8$\times$} \\
\bottomrule
\end{tabular}
\end{table}

\subsection{Ablation and Sensitivity}
\label{sec:ablation}

We next test whether the stopping rule remains reliable beyond the main experimental setting. Starting from the reference configuration used above, we vary one factor at a time while holding the others fixed. We study robustness across self-evolving systems and underlying models, and then examine sensitivity to the validation-set size and detector parameters.

\textbf{Across self-evolving systems.} We first test whether the stopping rule transfers across different self-evolving algorithms. Holding the underlying model and benchmarks fixed, we replace SkillOpt with GEPA, whose update and selection mechanism differs substantially from SkillOpt. We extract the same type of per-item paired evaluation outcomes from GEPA and apply the detector without changing its configuration. As shown in \Cref{tab:ablation-algo}, the detector raises alarms at similar stages of the evolution process under both systems and still removes a large fraction of the full-run computation. On SearchQA, GEPA stops at round~5 with $\hat{\nu}=2$ and saves 73.8\% of the estimated token cost; on GSM8K, it stops at round~6 with $\hat{\nu}=3$ and saves 65.5\%. Although GEPA runs are shorter overall, the same stopping rule continues to identify a low-improvement regime well before the full budget is exhausted. GEPA token usage is estimated from metric-call counts; details are provided in Appendix~\ref{app:gepa-align}.

\textbf{Across models.} We next test whether the stopping rule remains effective across different underlying models. Holding SkillOpt, the benchmark, and the detector configuration fixed, we compare DeepSeek V4 Flash, Qwen3-32B, and GPT-5.6 Luna. As shown in \Cref{tab:ablation-model,tab:ablation-model-gsm8k}, the detector adapts to different evolution speeds rather than stopping at a fixed round. On GSM8K, DeepSeek V4 Flash and GPT-5.6 Luna both saturate quickly and trigger an alarm at round~4, whereas Qwen3-32B continues improving longer and stops at round~8. A similar pattern appears on SearchQA, where Qwen3-32B stops later than DeepSeek V4 Flash. Despite these differences in stopping time, the detector still removes a large fraction of the full-run computation across all model--benchmark combinations, with token savings ranging from 68.9\% to 92.6\%.

\begin{table}[t]
\centering
\small
\setlength{\tabcolsep}{4pt}
\renewcommand{\arraystretch}{0.88}
\caption{Robustness across self-evolving systems using DeepSeek V4 Flash and the same detector configuration. We compare SkillOpt and GEPA on the same benchmarks, reporting stopping behavior, computation cost, and unseen-test performance. GEPA token counts are estimated from metric-call counts.}
\label{tab:ablation-algo}
\begin{tabular}{lcccc}
\toprule
& \multicolumn{2}{c}{\textbf{SearchQA}} & \multicolumn{2}{c}{\textbf{GSM8K}} \\
\cmidrule(lr){2-3}\cmidrule(lr){4-5}
& SkillOpt & GEPA & SkillOpt & GEPA \\
\midrule

\multicolumn{5}{l}{\textit{Stopping behavior}} \\
Alarm round $T_{\mathrm{alarm}}$
    & 4 & 5 & 4 & 6 \\
Change point $\hat{\nu}$
    & 2 & 2 & 2 & 3 \\

\midrule

\multicolumn{5}{l}{\textit{Computation cost}} \\
Tokens to alarm (M)
    & 2.18 & 2.36 & 1.21 & 1.79 \\
Full-run tokens (M)
    & 25.85 & 8.99 & 11.86 & 5.18 \\
Tokens saved
    & \textbf{91.6\%} & \textbf{73.8\%}
    & \textbf{89.8\%} & \textbf{65.5\%} \\

\midrule

\multicolumn{5}{l}{\textit{Unseen-test performance}} \\
Baseline accuracy (\%)
    & 77.21 & 76.86 & 84.00 & 84.00 \\
Full-run selected (\%)
    & 82.00 & 82.36 & 98.20 & 97.20 \\
Returned accuracy (\%)
    & 82.43 & 82.43 & 97.40 & 97.20 \\
$\Delta$ Accuracy (pp)
    & +0.43 & +0.07 & -0.80 & 0.00 \\
95\% CI of $\Delta$ (pp)
    & $[-1.10,\,1.96]$ & $[-1.17,\,1.32]$ 
    & $[-1.91,\,0.31]$ &$[0.00,\,0.00]$ \\

\bottomrule
\end{tabular}
\end{table}

\begin{table}[t]
\centering
\small
\setlength{\tabcolsep}{4pt}
\renewcommand{\arraystretch}{0.88}
\caption{Robustness across underlying models on SearchQA using SkillOpt with the same detector configuration. The detector adapts its stopping time to different evolution dynamics while reducing computation and retaining comparable unseen-test performance.}
\label{tab:ablation-model}
\begin{tabular}{lccc}
\toprule
& DeepSeek V4 Flash & Qwen3-32B & GPT-5.6 Luna \\
\midrule

\multicolumn{4}{l}{\textit{Stopping behavior}} \\
Validation size $n$
    & 200 & 200 & 200 \\
Full-run budget $N$
    & 40 & 40 & 40 \\
Alarm round $T_{\mathrm{alarm}}$
    & 4 & 12 & 10 \\
Change point $\hat{\nu}$
    & 2 & 9 & 3 \\

\midrule

\multicolumn{4}{l}{\textit{Computation cost}} \\
Tokens to alarm (M)
    & 2.18 & 8.08 & 5.30 \\
Full-run tokens (M)
    & 25.85 & 32.68 & 17.03 \\
Tokens saved
    & \textbf{91.6\%} & \textbf{75.3\%} & \textbf{68.9\%} \\

\midrule

\multicolumn{4}{l}{\textit{Unseen-test performance}} \\
Baseline accuracy (\%)
    & 77.21 & 66.36 & 69.79 \\
Full-run selected (\%)
    & 82.00 & 72.93 & 81.07 \\
Returned accuracy (\%)
    & 82.43 & 73.14 & 80.07 \\
$\Delta$ Accuracy (pp)
    & +0.43 & +0.21 & -1.00 \\
95\% CI (pp)
    & $[-1.10,1.96]$
    & $[-1.45,1.88]$
    & $[-2.08,0.08]$ \\

\bottomrule
\end{tabular}
\end{table}

Additional sensitivity results are provided in Appendix~\ref{app:sensitivity}, including experiments with smaller validation sets and variations in the detector parameters $(\varepsilon,\delta)$. The detector tends to stop later with less validation data, while retaining substantial token savings and comparable unseen-test performance.

\section{Conclusion and Future Work}
\label{sec:discussion}

We introduced a stopping and output-selection procedure for self-evolving systems that uses sequential betting to determine when further evolution is no longer worthwhile and change-point estimation to select an earlier artifact for output. The procedure operates on paired evaluation outcomes already produced during evolution and requires no modification to the underlying system. Across multiple benchmarks, models, and self-evolving systems, it stops well before the fixed budget is exhausted while substantially reducing computation and retaining comparable unseen-test performance.

Future work could extend the localization guarantee to adaptive betting schemes such as aGRAPA, study data-driven ways to calibrate the stopping rule, and generalize the framework to richer evaluation signals beyond binary correctness. It would also be useful to examine how the method behaves under different validation-set sizes, noisier feedback, and more complex self-evolving systems in which proposals are generated and evaluated through longer interaction loops. More broadly, our results suggest that self-evolving systems need not rely on predetermined fixed budgets: sequential evidence can guide when to stop, while change-point estimation can help determine which earlier artifact to return.

\section*{Ethics Statement}

This work studies when to stop self-evolving AI systems in order to avoid unnecessary computation while preserving useful improvements. Our experiments use existing benchmarks and language models and do not involve human subjects or the collection of new personal data. The proposed stopping procedure may help reduce redundant model calls and associated computational cost. As with other evaluation-based stopping methods, its behavior depends on the representativeness of the validation data, and care should be taken when applying it in high-stakes settings.

\section*{Reproducibility Statement}

We provide the full stopping procedure, theoretical assumptions, and proofs in the main text and appendix. Experimental settings include the worthwhile-gain threshold $\varepsilon$, false-alarm control parameter $\delta$, betting constructions, validation-set sizes, model and benchmark configurations, and alarm-time comparison results. We also report both adaptive aGRAPA and fixed-mixture betting schemes, with their comparison provided in the appendix. Additional ablations cover models, evolving systems, validation-set sizes, and detector configurations.

\section*{AI Use Statement}

In this work, large language models (LLMs) were used solely for language refinement during manuscript preparation, including grammar correction, clarity improvement, and the identification of potentially ambiguous wording. All research ideas, methodological decisions, experiments, theoretical derivations, numerical results, and conclusions were independently checked and approved by the authors. The authors take full responsibility for the content of the paper.

\bibliographystyle{plainnat}
\bibliography{references}

% =====================================================================
\appendix

\section{Proofs}
\label{app:proofs}

We prove the two propositions in
\Cref{sec:theory}. Throughout, recall that
\[
Z_t
=
\frac{1}{n}\sum_{i=1}^{n}X_{t,i},
\]
and that the pre-change condition is
\begin{equation}
\label{eq:app-null}
\E[Z_t\mid\Fcal_{t-1}]
\ge
\varepsilon.
\end{equation}
We assume that $X_{t,1},\ldots,X_{t,n}$ are conditionally independent
given $\Fcal_{t-1}$ and that each betting fraction
$\lambda_{s,t}\in[0,\tfrac12]$ is predictable.

\subsection{Validity of Each Restart Wealth Process}

\begin{proof}
Fix a restart time $s$. For $t\ge s$, write the multiplicative
increment of the restart wealth as
\[
L_t^{(s)}
=
\prod_{i=1}^{n}
\left[
1+\lambda_{s,t}
\left(
\varepsilon-X_{t,i}
\right)
\right],
\]
so that
\[
W_t^{(s)}
=
W_{t-1}^{(s)}L_t^{(s)},
\qquad
W_{s-1}^{(s)}=1.
\]

We first verify non-negativity. Since
$X_{t,i}\in\{-1,0,1\}$,
$\varepsilon\in(0,1)$, and
$\lambda_{s,t}\in[0,\tfrac12]$,
\[
1+\lambda_{s,t}(\varepsilon-X_{t,i})
\ge
1+\frac12(\varepsilon-1)
=
\frac{1+\varepsilon}{2}
>
0.
\]
Hence $W_t^{(s)}$ is nonnegative.

Next, define
\[
\mu_{t,i}
=
\E[X_{t,i}\mid\Fcal_{t-1}].
\]
Because $\lambda_{s,t}$ is
$\Fcal_{t-1}$-measurable and the item-level signals are conditionally
independent given $\Fcal_{t-1}$,
\begin{align}
\E[L_t^{(s)}\mid\Fcal_{t-1}]
&=
\prod_{i=1}^{n}
\left[
1+\lambda_{s,t}
\left(
\varepsilon-\mu_{t,i}
\right)
\right].
\end{align}

All factors in the product are positive, so by the arithmetic--geometric
mean inequality,
\begin{align}
\E[L_t^{(s)}\mid\Fcal_{t-1}]
&\le
\left[
\frac{1}{n}
\sum_{i=1}^{n}
\left(
1+\lambda_{s,t}
(\varepsilon-\mu_{t,i})
\right)
\right]^n
\\
&=
\left[
1+\lambda_{s,t}
\left(
\varepsilon-
\frac{1}{n}\sum_{i=1}^{n}\mu_{t,i}
\right)
\right]^n.
\end{align}
Since
\[
\frac{1}{n}\sum_{i=1}^{n}\mu_{t,i}
=
\E[Z_t\mid\Fcal_{t-1}]
\ge
\varepsilon
\]
under the pre-change condition, we obtain
\[
\E[L_t^{(s)}\mid\Fcal_{t-1}]
\le
1.
\]

Therefore,
\begin{align}
\E[W_t^{(s)}\mid\Fcal_{t-1}]
&=
W_{t-1}^{(s)}
\E[L_t^{(s)}\mid\Fcal_{t-1}]
\\
&\le
W_{t-1}^{(s)}.
\end{align}
Thus $\{W_t^{(s)}\}_{t\ge s-1}$ is a nonnegative supermartingale
initialized at one.

Ville's inequality then gives, for any $h>0$,
\[
\Pr\left(
\sup_{t\ge s}W_t^{(s)}\ge h
\right)
\le
\frac{1}{h}.
\]
\end{proof}

\subsection{E-Detector Validity and False-Alarm Run-Length Guarantee}

We next show that combining the restart processes through the maximum
restart wealth produces a valid e-detector. This construction
corresponds to the CUSUM-style e-detector of
\citet{shin2023}.

\begin{proof}
Extend each restart process by setting
$W_t^{(s)}=1$ for $t<s$. By
\Cref{prop:validity}, each
$\{W_t^{(s)}\}_{t\ge0}$ is a nonnegative restart e-process under the
pre-change condition.

Recall that
\[
M_t
=
\max_{1\le s\le t}
W_t^{(s)}.
\]
For any stopping time $\tau$,
\[
M_\tau
\le
\sum_{s=1}^{\tau}
W_\tau^{(s)}.
\]
Taking expectations and writing the random upper limit using
indicators gives
\begin{align}
\E[M_\tau]
&\le
\sum_{s=1}^{\infty}
\E\left[
\ind\{s\le\tau\}
W_\tau^{(s)}
\right].
\end{align}

Because $\{s\le\tau\}\in\Fcal_{s-1}$ and
$W^{(s)}$ is a valid restart e-process, optional stopping gives
\[
\E\left[
W_\tau^{(s)}
\mid
\Fcal_{s-1}
\right]
\le
1.
\]
Therefore,
\begin{align}
\E\left[
\ind\{s\le\tau\}
W_\tau^{(s)}
\right]
&\le
\Pr(s\le\tau).
\end{align}
Summing over $s$,
\begin{align}
\E[M_\tau]
&\le
\sum_{s=1}^{\infty}
\Pr(s\le\tau)
\\
&=
\E[\tau].
\end{align}
Hence $M_t$ satisfies the defining e-detector condition
\[
\E[M_\tau]\le\E[\tau]
\]
for every stopping time $\tau$, and is therefore a valid e-detector
in the sense of \citet{shin2023}.

Now recall the alarm rule
\[
T_{\mathrm{alarm}}
=
\inf\left\{
t\ge1:
M_t\ge\frac{1}{\delta}
\right\}.
\]
If $T_{\mathrm{alarm}}$ is not almost surely finite, then
$\E[T_{\mathrm{alarm}}]=\infty$ and the desired result holds
immediately. Otherwise,
\[
M_{T_{\mathrm{alarm}}}
\ge
\frac{1}{\delta}.
\]
Applying the e-detector inequality at
$\tau=T_{\mathrm{alarm}}$ gives
\[
\frac{1}{\delta}
\le
\E[M_{T_{\mathrm{alarm}}}]
\le
\E[T_{\mathrm{alarm}}],
\]
and therefore
\[
\E[T_{\mathrm{alarm}}]
\ge
\frac{1}{\delta}.
\]
Thus the stopping rule has average run length to false alarm of at
least $1/\delta$ whenever the pre-change condition continues to hold.
\end{proof}

\subsection{Population Characterization of the Change Point}
\label{app:changepoint-proof}
\begin{proof}[Proof of the population conclusion in Proposition~\ref{prop:changepoint}]
For this result, consider the fixed betting fraction $\lambda$ used in
\Cref{prop:changepoint}. Taking logarithms of the restart wealth gives
\[
\log W_T^{(s)}
=
\sum_{u=s}^{T}
\sum_{i=1}^{n}
\log\left(
1+\lambda(\varepsilon-X_{u,i})
\right).
\]
Let $g_u=n\E[Y_u]$, with $Y_u$ defined in
\Cref{eq:normalized-log-growth}. By the tower property, $g_u<0$ before
$\nu$ and $g_u>0$ from $\nu$ onward. Thus,
\begin{equation}
\label{eq:app-logwealth}
\E[\log W_T^{(s)}]
=
\sum_{u=s}^{T}g_u.
\end{equation}

Suppose first that $s<\nu$. Then
\begin{align}
\E[\log W_T^{(s)}]
&=
\sum_{u=s}^{\nu-1}g_u
+
\sum_{u=\nu}^{T}g_u
\\
&<
\sum_{u=\nu}^{T}g_u
=
\E[\log W_T^{(\nu)}],
\end{align}
because $g_u<0$ for all $u<\nu$.

Now suppose that $s>\nu$. Then
\begin{align}
\E[\log W_T^{(s)}]
&=
\sum_{u=s}^{T}g_u
\\
&=
\sum_{u=\nu}^{T}g_u
-
\sum_{u=\nu}^{s-1}g_u
\\
&<
\sum_{u=\nu}^{T}g_u
=
\E[\log W_T^{(\nu)}],
\end{align}
because $g_u>0$ for all $u\ge\nu$.

Thus every $s\neq\nu$ has strictly smaller expected log restart wealth
than $s=\nu$, and hence
\[
\nu
=
\operatorname*{arg\,max}_{1\le s\le T}
\E[\log W_T^{(s)}].
\]

This is a population-level characterization for the common fixed-$\lambda$ setting considered here. For the corresponding realized restart wealth, the logarithm is strictly increasing, so
\[
\operatorname*{arg\,max}_{s}
W_T^{(s)}
=
\operatorname*{arg\,max}_{s}
\log W_T^{(s)}.
\]
Thus realized wealth and realized log wealth induce the same restart selection in this setting.
\end{proof}

\subsection{Change-Point Localization Error Bound}
\label{app:changepoint-error-proof}

\begin{proof}[Proof of the localization bound in Proposition~\ref{prop:changepoint-error}]
Recall that
\[
Y_t=\frac{1}{n}\sum_{i=1}^n V_{t,i},
\qquad
V_{t,i}=\log\left(1+\lambda(\varepsilon-X_{t,i})\right),
\]
and therefore
\[
\log W_T^{(s)}=n\sum_{u=s}^T Y_u.
\]
By assumption, conditional on $\Fcal_{t-1}$, the centered variables $V_{t,i}-\E[V_{t,i}\mid\Fcal_{t-1}]$ are independent and $\sigma^2$-sub-Gaussian. Hence,
\[
Y_t-\E[Y_t\mid\Fcal_{t-1}]
\]
is conditionally $\sigma^2/n$-sub-Gaussian.

First consider a starting round $s<\nu$. Since the processes initialized at $s$ and $\nu$ share all factors from round $\nu$ through $T$,
\[
\log W_T^{(s)}-\log W_T^{(\nu)}
=
n\sum_{u=s}^{\nu-1}Y_u.
\]
Let $k=\nu-s$. Then
\[
W_T^{(s)}\ge W_T^{(\nu)}
\quad\Longleftrightarrow\quad
\sum_{u=s}^{\nu-1}Y_u\ge0.
\]
Since $\E[Y_u\mid\Fcal_{u-1}]\le-\kappa$ for $u<\nu$,
\[
\sum_{u=s}^{\nu-1}\E[Y_u\mid\Fcal_{u-1}]
\le-k\kappa.
\]
Thus, defining
\[
D_u=Y_u-\E[Y_u\mid\Fcal_{u-1}],
\]
the event $\sum_{u=s}^{\nu-1}Y_u\ge0$ implies
\[
\sum_{u=s}^{\nu-1}D_u\ge k\kappa.
\]
Because each $D_u$ is conditionally $\sigma^2/n$-sub-Gaussian, the sum over $k$ rounds is conditionally $k\sigma^2/n$-sub-Gaussian. Therefore,
\[
\Pr\left(W_T^{(s)}\ge W_T^{(\nu)}\right)
\le
\exp\left(-\frac{nk\kappa^2}{2\sigma^2}\right).
\]

Now consider $s>\nu$. In this case,
\[
\log W_T^{(s)}-\log W_T^{(\nu)}
=
-n\sum_{u=\nu}^{s-1}Y_u.
\]
Let $k=s-\nu$. Then
\[
W_T^{(s)}\ge W_T^{(\nu)}
\quad\Longleftrightarrow\quad
\sum_{u=\nu}^{s-1}Y_u\le0.
\]
Since $\E[Y_u\mid\Fcal_{u-1}]\ge\kappa$ for $u\ge\nu$,
\[
\sum_{u=\nu}^{s-1}\E[Y_u\mid\Fcal_{u-1}]
\ge k\kappa.
\]
Hence the event $\sum_{u=\nu}^{s-1}Y_u\le0$ implies
\[
\sum_{u=\nu}^{s-1}D_u\le-k\kappa.
\]
Applying the lower-tail sub-Gaussian bound gives
\[
\Pr\left(W_T^{(s)}\ge W_T^{(\nu)}\right)
\le
\exp\left(-\frac{nk\kappa^2}{2\sigma^2}\right).
\]

If $|\hat{\nu}_T-\nu|\ge m$, then there exists some starting round $s$ with $|s-\nu|\ge m$ such that $W_T^{(s)}\ge W_T^{(\nu)}$. Applying a union bound over starting points on both sides of $\nu$ gives
\[
\Pr\left(|\hat{\nu}_T-\nu|\ge m\right)
\le
2\sum_{k=m}^{\infty}
\exp\left(-\frac{nk\kappa^2}{2\sigma^2}\right).
\]
Evaluating the geometric series yields
\[
\Pr\left(|\hat{\nu}_T-\nu|\ge m\right)
\le
\frac{
2\exp\left(-\frac{nm\kappa^2}{2\sigma^2}\right)
}{
1-\exp\left(-\frac{n\kappa^2}{2\sigma^2}\right)
}.
\]

To obtain the equivalent high-probability form, let $\alpha\in(0,1)$ and choose $m$ such that
\[
\frac{
2\exp\left(-\frac{nm\kappa^2}{2\sigma^2}\right)
}{
1-\exp\left(-\frac{n\kappa^2}{2\sigma^2}\right)
}
\le\alpha.
\]
Solving for $m$ gives
\[
m
\ge
\frac{2\sigma^2}{n\kappa^2}
\log\left(
\frac{2}
{\alpha\left(1-\exp\left(-\frac{n\kappa^2}{2\sigma^2}\right)\right)}
\right).
\]
Let $r_\alpha$ denote the right-hand side and choose $m=\lceil r_\alpha\rceil$. Since $|\hat\nu_T-\nu|$ is integer-valued, the events $\{|\hat\nu_T-\nu|\ge r_\alpha\}$ and $\{|\hat\nu_T-\nu|\ge m\}$ coincide. Therefore, with probability at least $1-\alpha$,
\[
|\hat{\nu}_T-\nu|
<
\frac{2\sigma^2}{n\kappa^2}
\log\left(
\frac{2}
{\alpha\left(1-\exp\left(-\frac{n\kappa^2}{2\sigma^2}\right)\right)}
\right),
\]
which proves the stated result.
\end{proof}

\section{Betting fraction and detectability floor}
\label{app:betting}

By \Cref{prop:validity}, any predictable $\lambda_{s,t}\in[0,\tfrac12]$ yields a valid restart wealth process. We instantiate the betting rule in two ways: adaptive aGRAPA and a fixed mixture.

\paragraph{Fixed mixture.}
Fix a finite grid
$\Lambda=\{\lambda^{(1)},\dots,\lambda^{(m)}\}\subset[0,\tfrac12]$
and nonnegative weights $w_j$ summing to $1$; we use uniform weights over
$\Lambda=\{0.1,0.2,0.4\}$. For each start $s$, run one wealth process per
grid value and average them:
\[
\overline{\Wealth}_t^{(s)}
=
\sum_j w_j\,\Wealth_t^{(s,\lambda^{(j)})}.
\]
For the fixed-mixture betting scheme, this averaged wealth serves as the
restart wealth $W_t^{(s)}$ used by the stopping and output-selection
procedure.

\begin{proposition}[The mixture is valid]
\label{prop:mixture}
If each $\Wealth_t^{(s,\lambda^{(j)})}$ is a nonnegative supermartingale under $H_0$, then so is $\overline{\Wealth}_t^{(s)}$, and the false-alarm bound of \Cref{prop:validity} applies with the same threshold $1/\delta$.
\end{proposition}

\begin{proof}
Non-negativity is immediate. By linearity of conditional expectation,
\[
\E\bigl[\overline{\Wealth}_t^{(s)} \mid \Fcal_{t-1}\bigr]
= \sum_j w_j\, \E\bigl[\Wealth_t^{(s,\lambda^{(j)})} \mid \Fcal_{t-1}\bigr]
\le \sum_j w_j\, \Wealth_{t-1}^{(s,\lambda^{(j)})}
= \overline{\Wealth}_{t-1}^{(s)},
\]
and $\overline{\Wealth}_{s-1}^{(s)} = \sum_j w_j = 1$. A convex combination of nonnegative supermartingales is a nonnegative supermartingale.
\end{proof}

The mixture requires no online estimation and remains a valid supermartingale regardless of sample size.

\paragraph{Growth-optimal plug-in (aGRAPA).}
A Kelly gambler maximizes $\E[\log(1 + \lambda G_{t,i})]$; a second-order expansion yields the approximate maximizer $\lambda^\star = \mu/(\sigma^2 + \mu^2)$ with $\mu = \E[G_{t,i}]$ and $\sigma^2 = \operatorname{Var}(G_{t,i})$. Replacing the moments by predictable plug-ins over $[s, t-1]$ and truncating to $[0,\tfrac12]$ gives, for a gambler started at $s$,
\[
\lambda_{s,t} = \min\!\left(
\frac{\max(\hat\mu_{s,t-1},\, 0)}{\hat\sigma^2_{s,t-1} + \hat\mu_{s,t-1}^2},\ \tfrac12 \right),
\quad
\hat\mu_{s,t-1} = \eps - \bar Z_{s,t-1}, \quad
\hat\sigma^2_{s,t-1} = \bar q_{s,t-1} - \bar Z_{s,t-1}^2,
\]
where $\bar Z_{s,t-1} = \frac{1}{t-s}\sum_{u=s}^{t-1} Z_u$ and $\bar q_{s,t-1} = \frac{1}{t-s}\sum_{u=s}^{t-1}\frac{n_{01}(u)+n_{10}(u)}{n}$ is the mean discordant rate; on the empty window set $\lambda_{s,s} = \lambda_0$ for a fixed $\lambda_0 \in (0,\tfrac12]$. We set the ratio to zero when $\hat\sigma^2_{s,t-1}+\hat\mu_{s,t-1}^2=0$. The moment identities follow from $G = \eps - X$ with $X \in \{-1,0,+1\}$: $\E[G] = \eps - \E[X]$ gives $\hat\mu = \eps - \bar Z$, while $X^2 = \ind\{X \ne 0\}$ gives $\operatorname{Var}(G) = \operatorname{Var}(X) = \bar q - \bar Z^2$, with $\eps$ cancelling.

\begin{proposition}[aGRAPA is valid]
\label{prop:agrapa}
The fraction $\lambda_{s,t}$ above is predictable and lies in $[0,\tfrac12]$, hence the aGRAPA wealth process is a nonnegative supermartingale under $H_0$ by \Cref{prop:validity}.
\end{proposition}

\begin{proof}
The window $[s,t-1]$ excludes round $t$, so $\hat\mu_{s,t-1}$ and $\hat\sigma^2_{s,t-1}$ are $\Fcal_{t-1}$-measurable and $\lambda_{s,t}$ is predictable. The inner $\max(\cdot, 0)$ and outer $\min(\cdot, \tfrac12)$ force $\lambda_{s,t} \in [0,\tfrac12]$, with $\lambda_{s,t} = 0$ when the denominator vanishes. Validity follows from \Cref{prop:validity}.
\end{proof}

\paragraph{Detectability floor.}

\begin{proposition}[Approximate plateau length required to fire]
\label{prop:floor}
On a pure plateau with $Z_u=0$, so every item pays $G_{u,i}=\eps$, a single gambler with constant fraction $\lambda$ multiplies its wealth by $(1+\lambda\eps)^n$ per round. Starting from wealth of order $1$ at the onset of the plateau, the approximate number of consecutive plateau rounds required to reach $1/\delta$ is
\[
T_{\mathrm{need}}
\approx
\frac{\log(1/\delta)}
{n\log(1+\lambda\eps)}.
\]
For the mixture over a grid $\Lambda$, this becomes
\[
T_{\mathrm{need}}
\approx
\frac{\log(1/\delta)+\log|\Lambda|}
{n\log(1+\lambda_{\max}\eps)},
\qquad
\lambda_{\max}=\max\Lambda.
\]
\end{proposition}

\begin{proof}
On a plateau round every item pays $\eps$, so the per-round wealth factor of a constant-$\lambda$ gambler is $(1 + \lambda\eps)^n$ and after $T$ rounds the log-wealth grows by $T\,n\log(1 + \lambda\eps)$. Setting this equal to $\log(1/\delta)$ and solving for $T$ gives the first display; for the mixture, the fastest-growing grid value dominates up to a factor $|\Lambda|$, adding $\log|\Lambda|$.
\end{proof}

\begin{corollary}[Resolution scales with the validation set]
\label{cor:double}
For small $\lambda\eps$, $\log(1 + \lambda\eps) \approx \lambda\eps$, so $T_{\mathrm{need}} \approx (\log(1/\delta) + \log|\Lambda|)/(n\lambda_{\max}\eps)$. For a fixed target plateau length the smallest resolvable $\eps$ scales like $1/n$: doubling $\Dval$ halves the worthwhile gain the detector can resolve.
\end{corollary}

When $n$ or $\eps$ is small relative to the available plateau length, the detector correctly stays silent. This is not a failure: validity (\Cref{prop:validity}) holds regardless of $n$, and only power degrades.

\section{Paired Unseen-Test Comparison}
\label{app:paired-test}

We compare the artifact returned by our stopping procedure and the artifact selected by the full run on the same unseen-test examples. Because each example produces a paired binary correctness outcome, the observations can be summarized by the $2\times2$ table

\[
\begin{array}{c|cc}
& \text{Full incorrect} & \text{Full correct} \\
\hline
\text{Returned incorrect} & n_{00} & n_{01} \\
\text{Returned correct}   & n_{10} & n_{11}
\end{array}.
\]

Here, $n_{01}$ counts examples answered incorrectly by the returned artifact but correctly by the full-run artifact, while $n_{10}$ counts examples on which the reverse occurs.

\paragraph{Paired accuracy difference.}
For $N=n_{00}+n_{01}+n_{10}+n_{11}$ paired examples, the difference in test accuracy is

\[
\widehat{\Delta}_{\mathrm{acc}}
=
\widehat{\mathrm{Acc}}_{\mathrm{returned}}
-
\widehat{\mathrm{Acc}}_{\mathrm{full}}
=
\frac{n_{10}-n_{01}}{N}.
\]

Thus, positive values favor the returned artifact and negative values favor the full-run artifact. Let

\[
\hat q=\frac{n_{01}+n_{10}}{N}
\]

denote the fraction of discordant pairs. Since the per-example paired difference takes values in $\{-1,0,+1\}$, an estimated standard error is

\[
\widehat{\mathrm{SE}}
=
\sqrt{
\frac{
\hat q-\widehat{\Delta}_{\mathrm{acc}}^{\,2}
}{N}
}.
\]

We report the corresponding normal-approximation 95\% confidence interval

\[
\widehat{\Delta}_{\mathrm{acc}}
\pm
1.96\,\widehat{\mathrm{SE}}.
\]

The paired accuracy difference quantifies both the magnitude and uncertainty of the performance change. We report it together with its 95\% confidence interval in \Cref{tab:main}.
\begin{table}[t]
\centering
\caption{Robustness across underlying models on GSM8K using SkillOpt with the same detector configuration.}

\label{tab:ablation-model-gsm8k}

\setlength{\tabcolsep}{4pt}
\begin{tabular}{lccc}
\toprule
& DeepSeek & Qwen & Luna \\
\midrule

\multicolumn{4}{l}{\textit{Stopping behavior}} \\
Validation size $n$
    & 200 & 200 & 200 \\
Full-run budget $N$
    & 40 & 40 & 40 \\
Alarm round $T_{\mathrm{alarm}}$
    & 4 & 8 & 4 \\
Change point $\hat{\nu}$
    & 2 & 2 & 2 \\

\midrule

\multicolumn{4}{l}{\textit{Computation cost}} \\
Tokens to alarm (M)
    & 1.21 & 2.79 & 0.84 \\
Full-run tokens (M)
    & 11.86 & 18.04 & 11.33 \\
Tokens saved
    & \textbf{89.8\%} & \textbf{84.5\%} & \textbf{92.6\%} \\

\midrule

\multicolumn{4}{l}{\textit{Unseen-test performance}} \\
Baseline accuracy (\%)
    & 84.00 & 78.80 & 24.40 \\
Full-run selected (\%)
    & 98.20 & 95.80 & 95.80 \\
Returned accuracy (\%)
    & 97.40 & 94.20 & 94.60 \\
$\Delta$ Accuracy (pp)
    & -0.80 & -1.60 & -1.20 \\
95\% CI (pp)
    & $[-1.91,0.31]$
    & $[-3.35,0.15]$
    & $[-2.86,0.46]$ \\

\bottomrule
\end{tabular}
\end{table}

\section{GEPA alignment and token estimation}
\label{app:gepa-align}

\paragraph{Seed alignment.}
SkillOpt begins from a fixed skill-document template with an empty skill body, whereas GEPA begins from an initial prompt. To make the two systems' starting points comparable, we align the GEPA seed to SkillOpt's template: the seed prompt carries the same task framing, with an empty-skill placeholder in place of SkillOpt's initial skill body. This ensures that round $t{=}0$ denotes a comparable ``no learned content'' state in both systems, so that alarm times and $\hat\nu$ are measured from the same origin.

\paragraph{Per-item outcomes.}
GEPA's gate selects along a Pareto front over minibatch subscores rather than scoring on a single full-validation pass. We read the per-item outcomes $b_{t,i}, c_{t,i}$ that the detector consumes directly from GEPA's saved validation-subscore matrix, which records each candidate's per-item correctness on the shared validation split.

\paragraph{Token estimation.}
SkillOpt logs token usage per agent call, so its token counts are exact. GEPA's public logs record the number of metric calls (LM inferences over the validation set) made during discovery, but not a per-call token breakdown. We therefore estimate GEPA token usage as
\[
\text{tokens} \;\approx\; (\text{number of metric calls}) \times (\text{average tokens per call}),
\]
with the average tokens per call derived from the matched SkillOpt--DeepSeek rollout on the same benchmark: $2249$ tokens/call on SearchQA and $1354$ tokens/call on GSM8K. Because a metric call is an LM inference over the same validation set in both systems, the per-call cost is comparable across systems, so this estimate is accurate up to differences in prompt length between the two seeds. The GEPA token figures in \Cref{tab:ablation-algo} should be read as estimates on this basis; the alarm and $\hat\nu$ values, which depend only on the per-item outcomes, are exact.

\section{Alarm-time comparison}
\label{app:full-alarms}

\Cref{tab:full-alarms} compares the alarm time and change-point estimate under the two betting schemes used in our experiments: the fixed mixture and aGRAPA. Across all matched settings, the two schemes produce the same $\hat\nu$, while the alarm time differs by at most a few rounds. The main text reports the aGRAPA results for conciseness.

\begin{table}[t]
\centering
\small
\renewcommand{\arraystretch}{1.15}
\caption{Alarm\,/\,$\hat\nu$ for a single configuration ($\varepsilon=0.01$, $\delta=0.05$) across SkillOpt and GEPA settings. ``mix'' denotes the fixed mixture with $\Lambda=\{0.1,0.2,0.4\}$, and ``aGRAPA'' denotes the adaptive growth-optimal plug-in.}
\label{tab:full-alarms}
\begin{tabular}{lllcc}
\toprule
System & Benchmark & Model & mix & aGRAPA \\
\midrule
\multirow{8}{*}{SkillOpt}
 & SearchQA ($n{=}200$) & DeepSeek V4 Flash & 5/2  & 4/2  \\
 & SearchQA ($n{=}200$) & Qwen3-32B         & 12/9 & 12/9 \\
 & SearchQA ($n{=}200$) & GPT-5.6 Luna      & 11/3 & 10/3 \\
 & SearchQA ($n{=}100$) & DeepSeek V4 Flash & 10/4 & 6/4  \\
 & GSM8K ($n{=}200$)    & DeepSeek V4 Flash & 5/2  & 4/2  \\
 & GSM8K ($n{=}200$)    & Qwen3-32B         & 9/2  & 8/2  \\
 & GSM8K ($n{=}200$)    & GPT-5.6 Luna      & 5/2  & 4/2  \\
 & GSM8K ($n{=}100$)    & DeepSeek V4 Flash & 8/2  & 5/2  \\
\midrule
\multirow{2}{*}{GEPA}
 & SearchQA & DeepSeek V4 Flash & 5/2 & 5/2 \\
 & GSM8K    & DeepSeek V4 Flash & 6/3 & 6/3 \\
\bottomrule
\end{tabular}
\end{table}

\section{Sensitivity Analysis}
\label{app:sensitivity}
\paragraph{Sensitivity to validation-set size.} We examine how the detector behaves when less validation data are available. Holding the self-evolving system, underlying model, and detector configuration fixed, we reduce the validation set size from $n=200$ to $n=100$. As shown in \Cref{tab:ablation-n}, the detector generally stops later with fewer validation examples, since each round provides less evidence. On SearchQA, the alarm moves from round~4 to round~6, while on GSM8K it moves from round~4 to round~5. Despite the smaller validation sets, the method still saves more than 85\% of the total token cost, and the paired unseen-test accuracy differences remain small, with 95\% confidence intervals containing zero.
\begin{table}[t]
\centering
\small
\caption{Sensitivity to validation-set size using SkillOpt with DeepSeek V4 Flash and the same detector configuration. Smaller validation sets provide less evidence per round and therefore lead to slightly later stopping, while substantial computation is still saved and unseen-test performance remains comparable.}
\label{tab:ablation-n}
\begin{tabular}{lcccc}
\toprule
& \multicolumn{2}{c}{\textbf{SearchQA}} & \multicolumn{2}{c}{\textbf{GSM8K}} \\
\cmidrule(lr){2-3}\cmidrule(lr){4-5}
& $n=200$ & $n=100$ & $n=200$ & $n=100$ \\
\midrule

\multicolumn{5}{l}{\textit{Stopping behavior}} \\[1mm]
Full-run budget $N$
    & 40 & 40 & 40 & 40 \\
Alarm round $T_{\mathrm{alarm}}$
    & 4 & 6 & 4 & 5 \\
Change point $\hat{\nu}$
    & 2 & 4 & 2 & 2 \\

\midrule

\multicolumn{5}{l}{\textit{Computation cost}} \\[1mm]
Tokens to alarm (M)
    & 2.18 & 2.26 & 1.21 & 0.83 \\
Full-run tokens (M)
    & 25.85 & 16.10 & 11.86 & 6.52 \\
Tokens saved
    & \textbf{91.6\%} & \textbf{85.9\%}
    & \textbf{89.8\%} & \textbf{87.3\%} \\

\midrule

\multicolumn{5}{l}{\textit{Unseen-test performance}} \\[1mm]
Baseline accuracy (\%)
    & 77.21 & 77.21 & 84.00 & 84.00 \\
Full-run selected (\%)
    & 82.00 & 82.43 & 98.20 & 97.20 \\
Returned accuracy (\%)
    & 82.43 & 81.79 & 97.40 & 97.60 \\
$\Delta$ Accuracy (pp)
    & +0.43 & -0.64 & -0.80 & +0.40 \\
95\% CI of $\Delta$ (pp)
    & $[-1.10,\,1.96]$
    & $[-1.79,\,0.50]$
    & $[-1.91,\,0.31]$
    & $[-0.71,\,1.51]$ \\

\bottomrule
\end{tabular}
\end{table}

\paragraph{Sensitivity to detector parameters.} We further examine sensitivity to the improvement threshold $\varepsilon$ and false-alarm control parameter $\delta$. \Cref{tab:sensitivity-params} reports the alarm round $T_{\mathrm{alarm}}$ and estimated change point $\hat{\nu}$ across a grid of parameter values. Larger values of $\varepsilon$ or $\delta$ generally lead to earlier stopping, while the estimated change point remains relatively stable across a broad range of settings.
\begin{table}[t]
\centering
\footnotesize
\setlength{\tabcolsep}{4pt}
\renewcommand{\arraystretch}{0.9}
\caption{Sensitivity of the alarm round $T_{\mathrm{alarm}}$ and estimated change point $\hat{\nu}$ to $(\varepsilon,\delta)$. Each entry reports $T_{\mathrm{alarm}}/\hat{\nu}$; ``--'' indicates that no alarm is raised within the evolution horizon.}
\label{tab:sensitivity-params}
\begin{minipage}[t]{0.48\textwidth}
\centering
\textbf{SearchQA} ($n=200$)

\begin{tabular}{lccc}
\toprule
$\varepsilon\backslash\delta$ & 0.01 & 0.05 & 0.10 \\
\midrule
0.002 & 10/2 & 4/2 & 4/2 \\
0.005 & 5/2 & 4/2 & 4/2 \\
0.01 & 4/2 & 4/2 & 3/2 \\
0.02 & 4/2 & 3/2 & 3/2 \\
0.05 & 3/2 & 3/2 & 3/2 \\
\bottomrule
\end{tabular}
\end{minipage}
\hfill
\begin{minipage}[t]{0.48\textwidth}
\centering
\textbf{GSM8K} ($n=200$)

\begin{tabular}{lccc}
\toprule
$\varepsilon\backslash\delta$ & 0.01 & 0.05 & 0.10 \\
\midrule
0.002 & 8/2 & 7/2 & 5/2 \\
0.005 & 7/2 & 5/2 & 5/2 \\
0.01 & 5/2 & 4/2 & 4/2 \\
0.02 & 4/2 & 3/2 & 3/2 \\
0.05 & 3/2 & 3/2 & 3/2 \\
\bottomrule
\end{tabular}
\end{minipage}

\vspace{6pt}

\begin{minipage}[t]{0.48\textwidth}
\centering
\textbf{LiveMath} ($n=18$)

\begin{tabular}{lccc}
\toprule
$\varepsilon\backslash\delta$ & 0.01 & 0.05 & 0.10 \\
\midrule
0.002 & -- & 12/2 & 8/2 \\
0.005 & -- & 12/2 & 7/2 \\
0.01 & -- & 8/2 & 7/2 \\
0.02 & 12/2 & 7/2 & 6/2 \\
0.05 & 8/2 & 6/2 & 5/2 \\
\bottomrule
\end{tabular}
\end{minipage}
\hfill
\begin{minipage}[t]{0.48\textwidth}
\centering
\textbf{OfficeQA} ($n=24$)

\begin{tabular}{lccc}
\toprule
$\varepsilon\backslash\delta$ & 0.01 & 0.05 & 0.10 \\
\midrule
0.002 & 9/5 & 8/5 & 7/5 \\
0.005 & 9/5 & 8/5 & 7/5 \\
0.01 & 9/5 & 8/5 & 7/5 \\
0.02 & 9/5 & 8/5 & 7/5 \\
0.05 & 8/4 & 6/4 & 6/4 \\
\bottomrule
\end{tabular}
\end{minipage}
\end{table}

\section{Additional Related Work}
\label{app:related-work}
\paragraph{Stopping rules for training and inference.}
Early stopping has long been used in model training to reduce computation and mitigate overfitting, typically by terminating optimization when validation performance ceases to improve according to a patience rule or threshold~\citep{prechelt2012,caruana2001}. Subsequent work has studied more principled, data-dependent stopping criteria for iterative learning procedures~\citep{yao2007,raskutti2014}. Related stopping mechanisms have also been explored in modern LLM systems, primarily at inference time. Semantic patience terminates iterative multi-agent refinement when successive outputs stabilize~\citep{shp}; statistical stopping rules have been proposed to limit overthinking within a reasoning trajectory~\citep{xie2026}; hidden-state diagnostics have been used to detect premature commitment during inference~\citep{mehta2026}; and always-valid e-processes have been applied to determine when a generate--verify procedure has accumulated sufficient evidence to release an output~\citep{cho2026release}. These approaches study when to terminate training or inference under their respective signals and objectives. In contrast, our setting concerns a self-evolving propose--evaluate--commit loop over persistent artifacts, where stopping must be determined from repeated paired validation outcomes and is coupled with the additional problem of selecting which previously retained artifact to return.

\clearpage
\section{Main Evolution Trajectories}
\label{app:main-trajectories}

\begin{figure}[htbp]
\centering
\includegraphics[width=\textwidth]{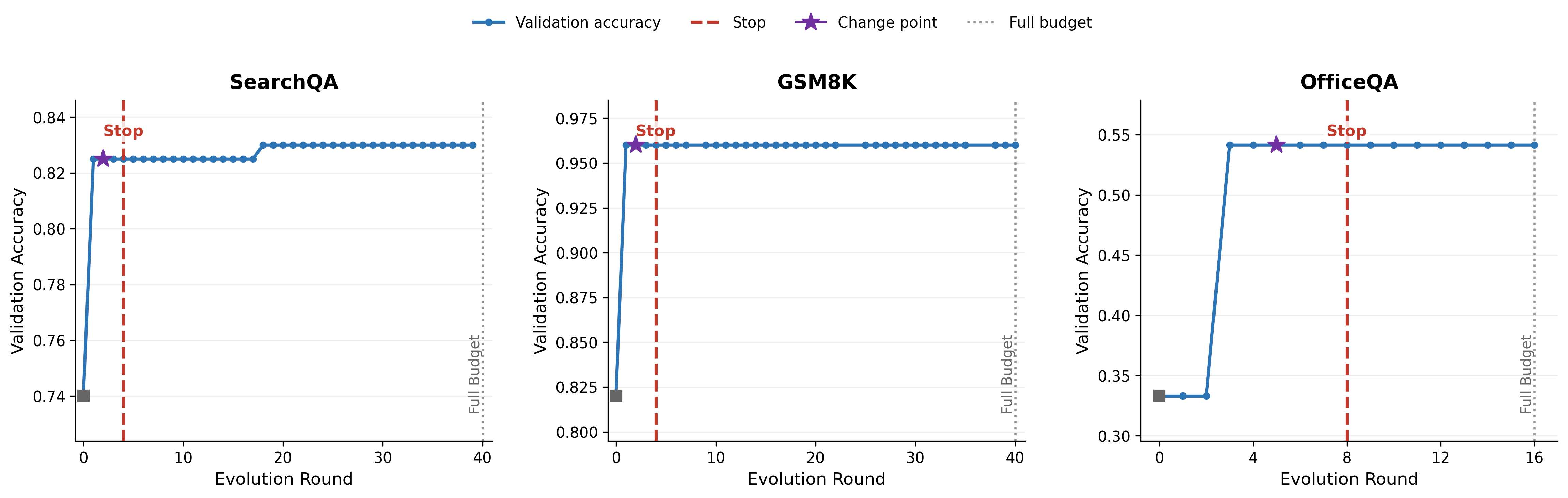}
\caption{Evolution trajectories under SkillOpt with DeepSeek V4 Flash on SearchQA, GSM8K, and OfficeQA. The red dashed line marks $T_{\mathrm{alarm}}$, the purple star marks $\hat{\nu}$, and the gray dotted line denotes the full evolution budget.}
\label{fig:main-trajectories}
\end{figure}

As shown in \Cref{fig:main-trajectories} and \Cref{tab:main}, the detector raises an alarm well before the full-run budget is exhausted on all three benchmarks. On SearchQA and GSM8K, it stops at round~4 out of 40, reducing token usage from 25.85M to 2.18M and from 11.86M to 1.21M, respectively. This corresponds to token savings of 91.6\% and 89.8\%. OfficeQA is substantially more expensive per evolution round, yet the detector still stops at round~8 out of 16 and reduces total token usage by 48.4\%. These results indicate that fixed iteration budgets can spend a substantial fraction of computation after the main performance gains have already occurred.

The estimated change points are $\hat\nu=2$ for SearchQA and GSM8K and $\hat\nu=5$ for OfficeQA. They precede the respective alarm rounds because the detector needs additional observations to accumulate sufficient stopping evidence.

\end{document}